\documentclass[runningheads]{llncs}
\usepackage[T1]{fontenc}
\usepackage{graphicx}
\usepackage{subcaption}
\usepackage{hyperref}
\usepackage{amsmath}
\usepackage{amssymb}
\usepackage{xcolor}
\usepackage{isomath}

\usepackage{amsthm}
\theoremstyle{plain}
\newtheorem{thm}{Theorem} % reset theorem numbering for each chapter
\newtheorem*{thm*}{Theorem} % reset theorem numbering for each chapter

\def\R{\mathbb{R}}

\def\relu{\mathrm{ReLU}}
\def\mlp{\mathrm{MLP}}
\def\1{\mathds{1}}
\def\our{SGCN}
\begin{document}
\title{Spatial Graph Convolutional Networks}  % \thanks{Supported by organization x.}
\titlerunning{Spatial Graph Convolutional Networks}
\author{Tomasz Danel\inst{1}\orcidID{0000-0001-6053-0028} \and
Przemys\l{}aw Spurek\inst{1}\orcidID{0000-0003-0097-5521} \and
Jacek Tabor\inst{1}\orcidID{0000-0001-6652-7727} \and 
Marek \'Smieja\inst{1}\orcidID{0000-0003-2027-4132} \and 
\L{}ukasz  Struski\inst{1}\orcidID{0000-0003-4006-356X} \and 
Agnieszka S\l{}owik\inst{2}\orcidID{0000-0001-7113-4098} \and 
\L{}ukasz Maziarka\inst{1}\orcidID{0000-0001-6947-8131}}
\authorrunning{T. Danel et al.}
\institute{Faculty of Mathematics and Computer Science, Jagiellonian University, \L{}ojasiewicza 6, 30-428 Krakow, Poland \and
Department of Computer Science and Technology, University of Cambridge,\\ 15 JJ Thomson Ave, Cambridge CB3 0FD, UK}
\maketitle              % typeset the header of the contribution
\begin{abstract}
Graph Convolutional Networks (GCNs) have recently become the primary choice for learning from graph-structured data, superseding hash fingerprints in representing chemical compounds. However, GCNs lack the ability to take into account the ordering of node neighbors, even when there is a geometric interpretation of the graph vertices that provides an order based on their spatial positions. To remedy this issue, we propose Spatial Graph Convolutional Network (\our{}) which uses spatial features to efficiently learn from graphs that can be naturally located in space. Our contribution is threefold: we propose a GCN-inspired architecture which (i) leverages node positions, (ii) is a proper generalization of both GCNs and Convolutional Neural Networks (CNNs), (iii) benefits from augmentation which further improves the performance and assures invariance with respect to the desired properties. Empirically, \our{} outperforms state-of-the-art graph-based methods on image classification and chemical tasks.

\end{abstract}
\section{Introduction}

Convolutional Neural Network (CNNs) use trainable filters to process images or grid-like objects in general. They have quickly overridden feed-forward networks in computer vision tasks, and later also excelled in parsing text data~\cite{kim2014convolutional}, thanks to the small number of parameters and the locality of the extracted features. The convolutional architectures were applied to numerous visual learning tasks on which they outperformed humans, e.g. image classification \cite{krizhevsky2012imagenet}, object detection \cite{seferbekov2018feature} or image captioning \cite{yang2017dense}. However, CNNs can only be used to analyze tensor data in which local patterns are anticipated, e.g. images, text, and time series. One of the common data structures that does not conform to this requirement is graph, which can be used to represent, e.g. social networks, neural networks, city maps, and chemical compounds. In these applications, local patters may also play a key role in processing big graph structures. Borrowing from CNNs, Graph Convolutionl Networks (GCNs) use local filters to aggregate information from neighboring nodes \cite{duvenaud2015convolutional,defferrard2016convolutional}. However, most of these networks do not distinguish node neighbors and apply the same weights to each of them, sometimes modified by node degrees \cite{kipf2016semi}, edge attributes, or trainable attention weights \cite{velivckovic2017graph}.

In many cases, graphs are coupled with spatial information embedded in their nodes. For example, images can be transformed to graphs where nodes correspond to image pixels (color channels). In this case, each pixel has 2-dimensional coordinates, which define its position in the image. In chemical applications, molecules can be represented as graphs constructed from their structural formulas (Figure~\ref{fig:representation}). Additionally, atoms (nodes in a molecular graph) organize themselves in the 3-dimensional space to reach the minimum energy state, and the shape of a compound is called a \textbf{molecular conformation}. Standard GCNs do not take spatial positions of the nodes into account, which is a considerable difference between GCNs and CNNs. Moreover, in the case of images, geometric features allow to augment data with translation or rotation and significantly enlarge the given dataset, which is crucial when the number of examples is limited.

\begin{figure}[tb]
\centering
  \hfill
  \begin{subfigure}[]{0.3\textwidth}
  \centering
    \includegraphics[width=\textwidth]{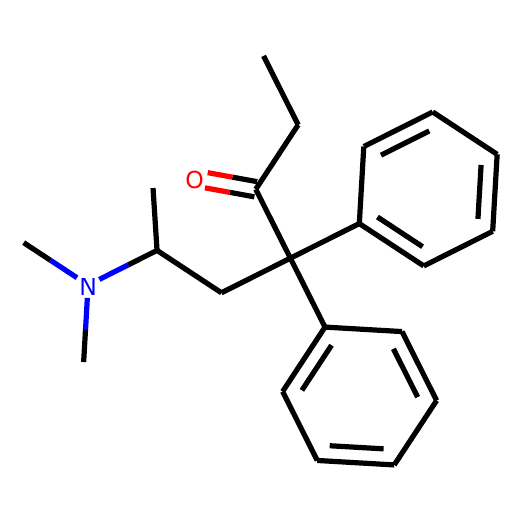}
    \caption{Structural formula}
    \label{fig:formula}
  \end{subfigure}
  \hfill
  \begin{subfigure}[]{0.3\textwidth}
  \centering
    \includegraphics[width=\textwidth]{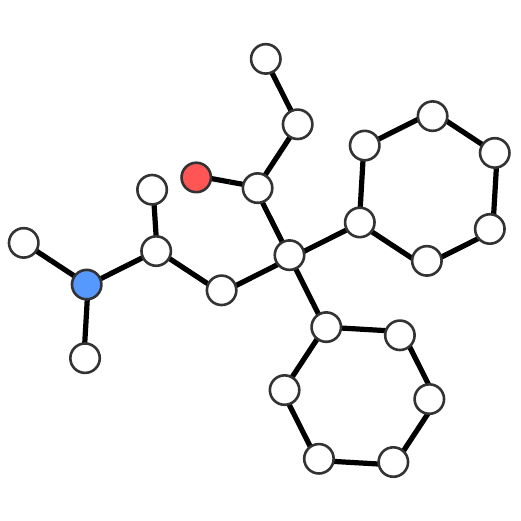}
    \caption{Molecular graph}
    \label{fig:graph}
  \end{subfigure}
  \hfill
  \begin{subfigure}[]{0.3\textwidth}
  \centering
    \includegraphics[width=\textwidth]{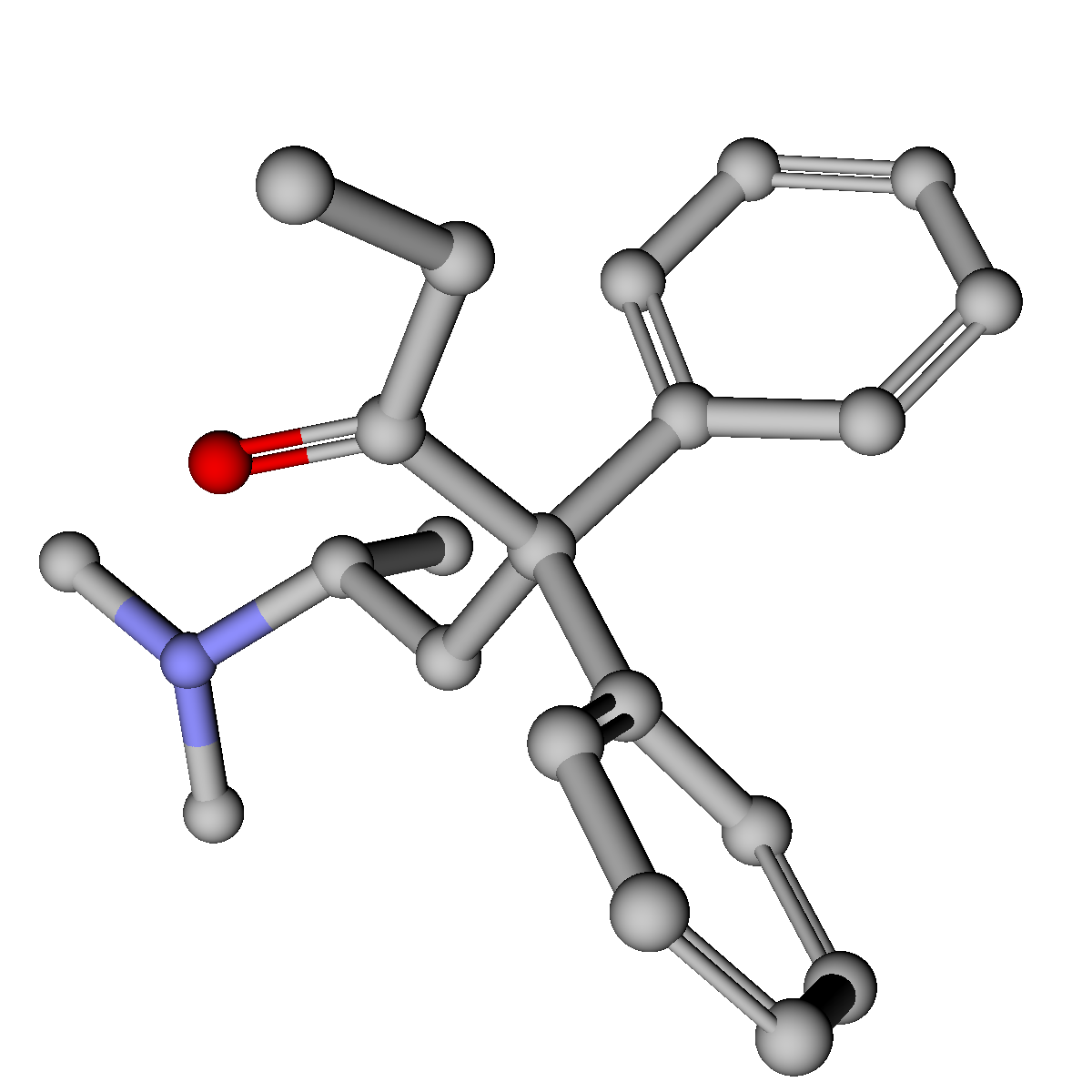}
    \caption{Exemplary conformation}
    \label{fig:conformation}
  \end{subfigure}
  \hfill \phantom{ }
  \caption{Representation of small molecules. 
  (a) shows the structural formula of a compound (methadone). This notation is commonly used by chemists. (b) presents the molecular graph constructed from the structural formula. The vertices denote atoms, and bonds are represented by the undirected edges. (c) depicts a molecular conformation (one of the energetic minima), which is a 3D embedding of the graph.}
  \label{fig:representation}
\end{figure}

%nasza metoda
In this paper, we propose Spatial Graph Convolutional Networks (\our{}), a variant of GCNs, which is a proper generalization of CNNs to the case of graphs. In contrast to existing GCNs, \our{} uses spatial features of nodes to aggregate information from the neighbors. On one hand, this geometric interpretation is useful to model many real examples of graphs, such as graphs of chemical compounds. In this case, it is possible to perform data augmentation by rotating a given graph in a spatial domain and, in consequence, improve network generalization when the amount of data is limited. Since typical graph convolutions do not take spatial features into account, data augmentation has not been possible in that case. On the other hand, a single layer of \our{} can be parametrized so that it returns an output identical to a standard convolutional layer on grid-like objects, such as images (see Theorem~\ref{thm:conv}). 

The proposed method was evaluated on various datasets and compared with the state-of-the-art methods. \our{} was applied to classify images represented as graphs. The proposed method was also tested on chemical benchmark datasets. Experiments demonstrate that combining spatial information with data augmentation leads to more accurate predictions.

% summary
Our contributions can be summarized as follows:
\begin{itemize}
    \item \our{} is proposed as a novel architecture of GCN that can effectively use additional geometric/spatial features to enhance a graph structure.
    \item The proposed architecture is a proper generalization of GCNs and CNNs, which is formally proven in Section~\ref{sec:theory}.
    \item In contrast to the existing approaches, \our{} gives more control over the geometric structure, allowing to exploit the spatial arrangement of graph nodes and to use data augmentation to achieve better performance through the selective invariance of spatial properties, such as node ordering, rotations, or translations. Note that \our{} surpasses classical GCNs when at least one spatial property, e.g. node ordering, is relevant for the given task.
\end{itemize}

\noindent The code is available at \href{https://github.com/gmum/geo-gcn}{github.com/gmum/geo-gcn}.

\section{Spatial graph convolutional network}

In this section, we introduce \our{}. First, we recall a basic construction of standard GCNs. Next, we present the intuition behind our approach and formally introduce \our{}. Finally, we briefly discuss practical advantages of spatial graph convolutions.

We use the following notation throughout this paper: let $\mathcal{G} = (V,\matrixsym{A})$ be a graph, where $V = \{v_1,\ldots,v_n\}$ denotes a set of nodes (vertices) and ${\matrixsym{A} = [a_{ij}]_{i,j=1}^n}$ represents edges. Let $a_{ij} = 1$ if there is a directed edge from $v_i$ to $v_j$, and $a_{ij} = 0$ otherwise. Each node $v_i$ is represented by a $d$-dimensional feature vector $\vectorsym{x}_i \in \R^d$. Typically, graph convolutional neural networks transform these feature vectors over multiple subsequent layers to produce the final prediction.

\subsection{Graph convolutions} 

Let $\matrixsym{H} = [\vectorsym{h}_1, \ldots, \vectorsym{h}_n]$ denote the matrix of node features being an input to a convolutional layer, where $\vectorsym{h}_i\in \mathbb{R}^{d_{in}}$ are column vectors. The dimension of $\vectorsym{h}_i$ is determined by the number of filters used in the previous layer. We denote as $\matrixsym{X} = [\vectorsym{x}_1,\ldots,\vectorsym{x}_n]$ the input representation for the first layer. 

A typical graph convolution is defined by combining two operations. For each node $v_i$, feature vectors of its neighbors $N_i = \{j: a_{ij} = 1\}$ are first aggregated:
\begin{equation}\label{eq:neigh}
    \vectorsym{\bar{h}}_i = \sum_{j \in N_i} u_{ij} \vectorsym{h}_j, % (or \bar{H} = U H^{T})
\end{equation}
which could be also written in a matrix form as $\matrixsym{\bar{H}} = \matrixsym{U}\matrixsym{H}^{T}$.
Where the weights % $u_{ij} \in \R$ 
$\matrixsym{U} \in \R^{n \times n}$ are either trainable (e.g. \cite{velivckovic2017graph} applied attention mechanism) or determined by adjacency matrix $\matrixsym{A}$ (e.g. \cite{kipf2016semi} motivated their selection using spectral graph theory).

Next, a standard MLP is applied to transform the intermediate representation $\matrixsym{\bar{H}} = [\vectorsym{\bar{h}}_1,\ldots,\vectorsym{\bar{h}}_n]$ into the final output of a given layer:
\begin{equation}\label{eq:mlp}
    \mlp(\matrixsym{\bar{H}}; \matrixsym{W}) = \relu(\matrixsym{W}^T \matrixsym{\bar{H}} + \vectorsym{b}),
\end{equation}
where $\matrixsym{W} \in \mathbb{R}^{d_{in}\times d_{out}}$ is a trainable weight matrix and $\vectorsym{b}\in\mathbb{R}^{d_{out}}$ is a trainable bias vector (added column-wise). A typical graph convolutional neural network is composed of a sequence of graph convolutional layers (described above), see Figure \ref{fig:geom-net}. Next, its output is aggregated to the final response depending on a given task, e.g. node or graph classification.

%%%%%%%%%%%%%%%%%%%%%%%%%%%%%%%%%%%%%%%%%%%%%%%%%%

\begin{figure}
\centering
\includegraphics[width=0.6\textwidth]{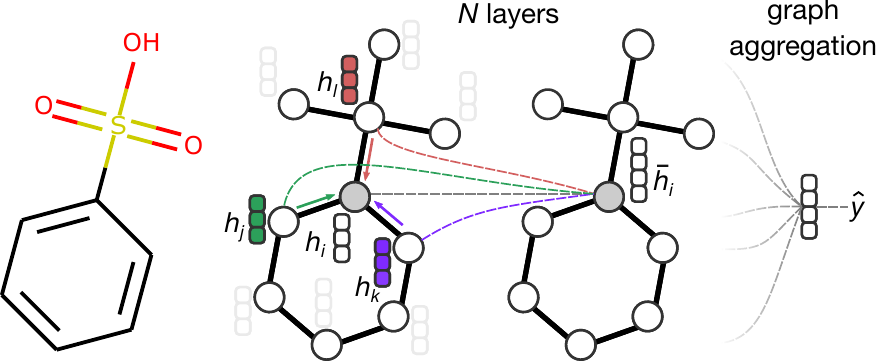} 
\caption{An overview of the full network. A molecule is transformed to the graph representation and fed to the $N$ consecutive (spatial) graph convolutional layers. In the figure, the convolution is demonstrated at the grey node -- feature vectors of the adjacent nodes $\vectorsym{h}_j$, $\vectorsym{h}_k$, and $\vectorsym{h}_l$ are aggregated together with the central node $\vectorsym{h}_i$ to create a new feature vector $\vectorsym{\bar{h}}_i$ for the grey node. In the proposed spatial variant, the relative positions of the neighbors are used in the aggregation (see Equation~\ref{eq:neigh-our}). At the end, all the node vectors are averaged, and the final prediction $\hat{y}$ is produced.}
\label{fig:geom-net}
\end{figure}

\subsection{Spatial graph convolutions}

In this section, the spatial graph convolutions are defined. The basic assumption is that each node $v_i$ is additionally identified by its coordinates $\vectorsym{p}_i \in \R^t$. In the case of images, $\vectorsym{p}_i$ is the vector of two dimensional pixel coordinates, while for chemical compounds, it denotes location of the atom in thr two or three dimensional space (depending on the representation of chemical compound). In contrast to standard features $\vectorsym{x}_i$, $\vectorsym{p}_i$ is not changed across layers, but only used to construct a better graph representation. For this purpose, \eqref{eq:neigh} is replaced by:
\begin{equation} \label{eq:neigh-our}
    \vectorsym{\bar{h}}_i(\matrixsym{U},\vectorsym{b}) = \sum_{j \in N_i} \relu(\matrixsym{U}^T(\vectorsym{p}_j - \vectorsym{p}_i) + \vectorsym{b}) \odot \vectorsym{h}_j,
\end{equation}
where $\matrixsym{U} \in \R^{t \times d}, \vectorsym{b} \in \R^d$ are trainable parameters, $d$ is the dimension of $\vectorsym{h_j}$ and 
$\odot$ is element-wise multiplication. The pair $\matrixsym{U},\vectorsym{b}$ plays a role of a convolutional filter which operates on the neighborhood of $v_i$. The relative positions in the neighborhood are transformed using a linear operation combined with non-linear ReLU function. This scalar is used to weigh the feature vectors $\vectorsym{h}_j$ in a neighborhood.

By the analogy with classical convolution, this transformation can be extended to multiple filters. Let $\tensorsym{U} = [\matrixsym{U}^{(1)},\ldots,\matrixsym{U}^{(k)}]$ and $\matrixsym{B} = [\vectorsym{b}^{(1)},\ldots,\vectorsym{b}^{(k)}]$ define $k$ filters. The intermediate representation $\vectorsym{\bar{h}}_i$ is then a vector defined by:
$$
\vectorsym{\bar{h}}_i (\tensorsym{U}, \matrixsym{B}) = \matrixsym{\bar{h}}_i(\matrixsym{U}^{(1)}, \vectorsym{b}^{(1)}) \oplus \dots \oplus \matrixsym{\bar{h}}_i(\matrixsym{U}^{(k)}, \vectorsym{b}^{(k)}),
$$
where $\oplus$ denotes the vector concatenation.
Finally, MLP transformation is applied in the same manner as in \eqref{eq:mlp} to transform these feature vectors into new representation.

Equation~\ref{eq:neigh-our} can be easily parametrized to obtain graph convolution presented in Equation~\ref{eq:neigh}. If all spatial features $\matrixsym{p}_i$ are put to 0, then \eqref{eq:neigh-our} reduces to:
$$
\vectorsym{\bar{h}}_i(\matrixsym{U},\vectorsym{b}) = \sum_{j \in N_i} \relu(\vectorsym{b}) \vectorsym{h}_j.
$$
This gives a vanilla graph convolution, where the aggregation over neighbors does not contain parameters. Different $\vectorsym{b} = \vectorsym{u}_{ij}$ can also be used for each pair of neighbors, which allows to mimic many types of graph convolutions.

\subsection{Data augmentation}

In practice, the number of training data is usually too small to provide sufficient generalization. To overcome this problem, one can perform data augmentation to produce more representative examples. In computer vision, data augmentation is straightforward and relies on rotating or translating the image. Nevertheless, in the case of classical graph structures, analogical procedure is difficult to apply. This is a serious problem in medicinal chemistry, where the goal is to predict biological activity based only on a small amount of verified compounds. The introduction of spatial features and our spatial graph convolutions allow us to perform data augmentation in a natural way, which is not possible using only the graph adjacency matrix.

The formula \eqref{eq:neigh-our} is invariant to the translation of spatial features, but its value depends on rotation of graph. In consequence, the rotation of the geometrical graph leads to different values of \eqref{eq:neigh-our}. Since in most domains the rotation does not affect the interpretation of the object described by such a graph (e.g. rotation does not change the chemical compound although one particular orientation may be useful when considering binding affinity, i.e. how well a given compound binds to the target protein), this property can be used to produce more instances of the same graph. This reasoning is exactly the same as in the classical view of image processing.

In addition, chemical compounds can be represented in many conformations. In a molecule, single bonds can rotate freely. Each molecule seeks to reach the minimum energy state, and thus some conformations are more probable to be found in nature than the other ones. Because there are multiple stable conformations, augmentation helps to learn only meaningful spatial relations. In some tasks, conformations may be included in the dataset, e.g. in binding affinity prediction, active conformations are those formed inside the binding pocket of a protein. Such a conformation can be discovered experimentally, e.g., through crystallization.

\section{Relation between \our{} and CNNs}
\label{sec:theory}

\begin{figure}
    \centering
  \hfill
  \begin{subfigure}[]{0.3\textwidth}
  \centering
    \includegraphics[width=\textwidth]{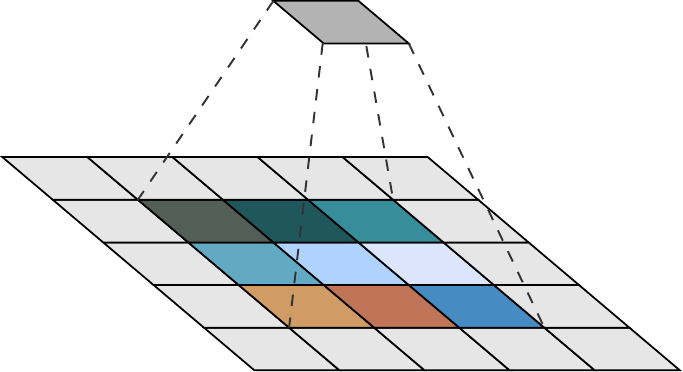}
    \caption{CNN}
    \label{fig:cnn-filter}
  \end{subfigure}
  \hfill
  \begin{subfigure}[]{0.3\textwidth}
  \centering
    \includegraphics[width=\textwidth]{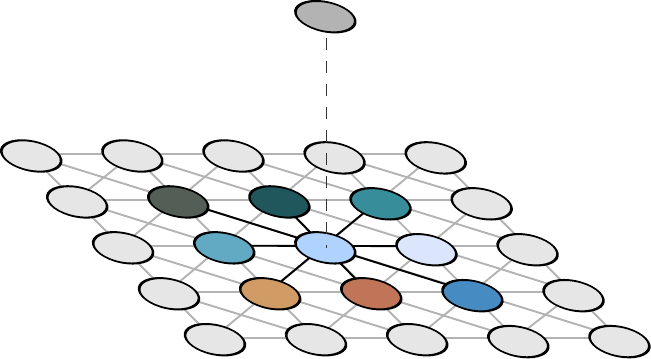}
    \caption{\our{}}
    \label{fig:geo-filter}
  \end{subfigure}
  \hfill
  \begin{subfigure}[]{0.3\textwidth}
  \centering
    \includegraphics[width=\textwidth]{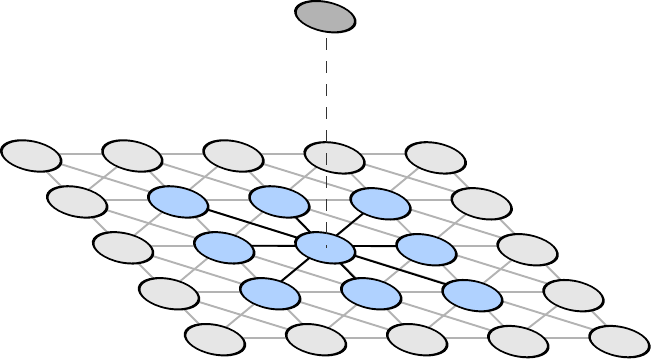}
    \caption{GCN}
    \label{fig:gcn-filter}
  \end{subfigure}
  \hfill \phantom{ }
  \caption{Comparison of different neural convolutional filters. Each color denotes different trainable weights. (a) shows convolutional filters for images, (b) shows our spatial graph convolutions, and (c) depicts graph convolutions.}
  \label{fig:filters}
\end{figure}

In contrast to typical GCNs, which consider graphs as relational structures, \our{} assumes that a graph can be coupled with a spatial structure, e.g. chemical compounds is a graph determined by atoms and bonds, which are embedded in 3D space. In particular, if we represent an image as a graph, where neighbor pixels are connected with edges, \our{} is capable of imitating the behavior of any CNNs operating on analogical image (see Figure~\ref{fig:filters}). In other words, the formula \eqref{eq:neigh-our} is constructed so that to parametrize any convolutional filter defined by classical CNNs. In this section, we first give a formal proof of this fact and next confirm this observation in an experimental study.

\paragraph{Theoretical findings.} Let us introduce a notation concerning convolutions in the case of images. For simplicity only convolutions without pooling and with odd mask size are considered. 
In general, given a filter $\matrixsym{F}=[f_{i'j'}]_{i',j' \in \{-k..k\}}$ its result on the image ${\matrixsym{H}=[h_{ij}]_{i \in \{1..N\},j \in \{1..K\}}}$ is given by
$$
\matrixsym{F}*\matrixsym{H}=\matrixsym{G}=[g_{ij}]_{i \in \{1..N\},j \in \{1..K\}},
$$
where
$$
g_{ij}=\sum_{\begin{array}{l} {}_{i'=-k..k:\, i+i'\in \{1..N\},} \\ {}_{j'=-k..k:\, j+j'\in \{1..K\}} \end{array}} f_{i'j'} h_{i+i',j+j'}.
$$

\begin{thm} \label{thm:conv}
Let $\matrixsym{H} \in \R^{N \times K}$ be an image. Let $\matrixsym{F}=[f_{i'j'}]_{i',j' \in \{-k..k\}}$ be a given convolutional filter, and let $n=(2k+1)^2$ (number of elements of $\matrixsym{F}$). Then there exist \our{} parameters: $\matrixsym{U} \in \R^{2 \times 1}$, $\vectorsym{b}_1,\ldots,\vectorsym{b}_n \in \R$, and $\vectorsym{w} \in \R^n$
such that the image convolution can be represented as \our{}, i.e.
$$
\matrixsym{F}*\matrixsym{H}=\sum_{i=1}^{n} \vectorsym{w}_i \matrixsym{\bar{H}}(\matrixsym{U},\vectorsym{b}_i).
$$
\end{thm}

\begin{proof}
Let $P \subset \R^2$ denote all possible positions in the convolutional filter $\matrixsym{F}$, i.e. $P=\left\{ [i', j']^T:i',j' \in \{-k,\ldots,k\}\right\}$.
Let $\vectorsym{u} \in \R^2$ denote an arbitrary vector which is not orthogonal to any element from $(P-P)$.
Let also consider that $\matrixsym{U} = [\vectorsym{u}]^T$ and $\vectorsym{b}_{i} = b_{i}$.
Then:
$$
\vectorsym{u}^T \vectorsym{p} \neq \vectorsym{u}^T \vectorsym{q}, \quad \text{ for } \vectorsym{p},\vectorsym{q} \in P: \  \vectorsym{p} \neq \vectorsym{q}.
$$
Consequently, the elements of $P$ may be ordered so that $\vectorsym{u}^T\vectorsym{p}_1 > \ldots > \vectorsym{u}^T \vectorsym{p}_n$. Let $\matrixsym{F}_i$ denote the convolutional filter, which has value one at the position $\vectorsym{p}_i$, and zero otherwise.

Let now choose $b_i$, such that:
$$
b_i \in (-\vectorsym{u}^T\vectorsym{p}_i,-\vectorsym{u}^T\vectorsym{p}_{i+1}) \  \text{ for } i<n; \quad
b_n > -\vectorsym{u}^T\vectorsym{p}_n.
$$

For example one may take:
$$
b_i=-\vectorsym{u}^T\frac{\vectorsym{p}_i+\vectorsym{p}_{i+1}}{2} \  \text{ for } i<n; \quad
b_n=-\vectorsym{u}^T\vectorsym{p}_n+1.
$$

Then observe that 
$$
\begin{array}{l}
\matrixsym{\bar{H}}(\matrixsym{U},\vectorsym{b}_1)=(\vectorsym{u}^T\vectorsym{p}_1+b_1) \cdot \matrixsym{F}_1 * \matrixsym{H} , \\
\matrixsym{\bar{H}}(\matrixsym{U},\vectorsym{b}_2)=\left[(\vectorsym{u}^T\vectorsym{p}_1+b_2) \cdot \matrixsym{F}_1+(\vectorsym{u}^T\vectorsym{p}_2+b_2) \cdot \matrixsym{F}_2\right] * \matrixsym{H},
\end{array}
$$
and generally for every $k=1..n$ we get
$$
\matrixsym{\bar{H}}(\matrixsym{U},\vectorsym{b}_k)=\left[\sum_{i=1}^k (\vectorsym{u}^T\vectorsym{p}_i+b_k) \matrixsym{F}_i\right] * \matrixsym{H},
$$
where all the coefficients in the above sum are strictly positive.

Consequently, 
$$
\matrixsym{F}_1*\matrixsym{H}=\frac{\matrixsym{\bar{H}}(\matrixsym{U},\vectorsym{b}_1)}{\vectorsym{u}^T\vectorsym{p}_1+b_1},
$$
and we obtain recursively that
$$
\matrixsym{F}_k*\matrixsym{H} =\tfrac{1}{\vectorsym{u}^T\vectorsym{p}_k+b_k}\matrixsym{\bar{H}}(\matrixsym{U},\vectorsym{b}_k)-
\tfrac{1}{\vectorsym{u}^T\vectorsym{p}_k+b_k}\sum_{i=1}^{k-1} (\vectorsym{u}^T\vectorsym{p}_i+b_k) \matrixsym{F}_i * \matrixsym{H},
$$
which trivially implies that every convolution $\matrixsym{F}_k * \matrixsym{H}$ can be obtained as a linear combination of $(\matrixsym{\bar{H}}(\matrixsym{U},\vectorsym{b_i}))_{i=1..k}$.

Since an arbitrary convolution $\matrixsym{F}=[f_{ij}]$ is given by $\matrixsym{F}=\sum_{i=1}^n f_{p_i} \matrixsym{F_i}$,
we obtain the assertion of the theorem.
\end{proof}

\paragraph{Experimental verification.} To experimentally demonstrate the correspondence between CNNs and \our{}, we consider the well-known MNIST dataset. To construct its graph representation, each pixel is mapped to a graph node, making a regular grid with connections between adjacent pixels. The node has 2-dimensional location, and it is characterized by a 1-dimensional pixel intensity. To show further capabilities of \our{}, we also consider an alternative representation \cite{monti2017geometric}, in which nodes are constructed from an irregular grid consisting of 75 superpixels. The edges are determined by spatial relations between nodes using k-nearest neighbors. 

For a comparison, we report the results from the literature by state-of-the-art methods used to process geometrical shapes: ChebNet \cite{defferrard2016convolutional}, MoNet \cite{monti2017geometric}, SplineCNN \cite{fey2018splinecnn} and GAT \cite{velivckovic2017graph}. In the first case of regular grid representation, \our{} is also compared to CNN with an analogical architecture, i.e. number of filters etc..

The results presented in Table \ref{tab:MNIST} show that \our{} outperforms comparable methods on both variants on MNIST dataset. Its performance is slightly better than SplineCNN, which reports state-of-the-art results on this task. We also get higher accuracy than CNN, which confirms experimentally that \our{} is its proper generalization.

\begin{table}
\normalsize
\caption{Classification accuracy on two graph representations of MNIST.}
\begin{center}
{\small
\begin{tabular}[width=\textwidth]{lcc}
\hline
{\bf Method} & {\bf Grid}  & {\bf Superpixels}  \\
\hline             
 CNN & 99.21\% & - \\
%  LeNet5 \cite{lecun1998gradient} & 99.33\% & -  \\
 ChebNet  & 99.14\% & 75.62\% \\
%  MoNet & 98.65\%  & 91.11\% \\
 MoNet & 99.19\%  & 91.11\% \\
 SplineCNN & 99.22\% & 95.22\% \\
 GAT & 99.26\% & 95.83\% \\
\our{} & {\bf 99.61\%} & {\bf 95.95\%} \\ 
\hline
\end{tabular}
}
\end{center}
\label{tab:MNIST}
\end{table}

\section{Experiments: a case study in predicting molecular properties of chemical compounds }

In this section, we take into account graphs representing chemical compounds and perform a large scale experimental verification on real-life tasks. 

\paragraph{Experimental setting.} Three datasets were chosen from MoleculeNet~\cite{molnet}, which is a benchmark for molecule-related tasks. Blood-Brain Barrier Permeability (BBBP) is a binary classification task of predicting whether or not a given compound is able to pass through the barrier between blood and the brain, allowing the drug to impact the central nervous system. Another two datasets, ESOL and FreeSolv, are solubility prediction tasks with continuous targets.

The datasets contain small molecules which are translated to the molecular graphs as it was presented in Figure~\ref{fig:representation}. At each node, there is a constructed feature vector describing the atom, which includes the atom type, hybridization, formal charge, number of heavy neighbors, number of attached hydrogens, whether or not the atom is in a ring, and whether or not it is aromatic. To predict atom positions in 3D space, we use universal force field (UFF) method from the RDKit package, which finds such a conformation of compound that minimizes the energy of molecule. Since UFF is not deterministic we run it a few times (up to 30) and additionally augment the data with random rotations. Datasets are split into train, validation and test subsets according to the  MoleculeNet standards. A random search is run for all models testing 100 hyperparameter sets for each of them. All runs are repeated 3 times.

\our{} is benchmarked against popular chemistry models: graph-based models (Graph Convolution~\cite{duvenaud2015convolutional}, Weave Model~\cite{weave}, and Message Passing Neural Network~\cite{gilmer2017neural}). We also use classical methods such as random forest and SVM, which do not operate on graph but rather they use a vector representation of chemical compound (ECFP fingerprint~\cite{ecfp} of size 1024). In addition, EAGCN~\cite{shang2018edge} and MAT~\cite{maziarka-mat} are included in the experiment, both of them using an attention mechanism. As for our method, the results are shown with train- and test-time augmentation of the data carried out in the manner described above\footnote{For all datasets, slight improvements can be observed with the augmented data.}. In order to investigate the impact of the positional features, the atom representation of the classical graph convolutional network is also enriched with the predicted atom positions, and the same procedure of augmentation is applied. We name this enriched architecture pos-GCN and include it in the comparison.

\begin{table}
\footnotesize
\caption{Performance on three chemical datasets measured with ROC AUC for BBBP and RMSE for ESOL and FreeSolv datasets. Best mean results and intervals overlapping with them are bolded.  For the first column higher is better, for the second and the third lower is better.}
\begin{center}
% {\small
\begin{tabular}[width=\textwidth]{lccc}
\hline
Method & BBBP & ESOL & FreeSolv \\
\hline             
SVM     & 0.603 $\pm$ 0.000  & 0.493 $\pm$ 0.000 & 0.391 $\pm$ 0.000 \\
RF      & 0.551 $\pm$ 0.005  & 0.533 $\pm$ 0.003 & 0.550 $\pm$ 0.004 \\
GC      & 0.690 $\pm$ 0.015  & 0.334 $\pm$ 0.017 & 0.336 $\pm$ 0.043 \\
Weave   & 0.703 $\pm$ 0.012  & 0.389 $\pm$ 0.045 & 0.403 $\pm$ 0.035 \\
MPNN    & 0.700 $\pm$ 0.019  & 0.303 $\pm$ 0.012 & \bf{0.299 $\pm$ 0.038} \\
EAGCN   & 0.664 $\pm$ 0.007  & 0.459 $\pm$ 0.019 & 0.410 $\pm$ 0.014 \\
MAT      & 0.711 $\pm$ 0.007  & 0.330 $\pm$ 0.002 & \bf{0.269 $\pm$ 0.007} \\
pos-GCN & 0.696 $\pm$  0.008  & 0.301 $\pm$ 0.011 & \bf{0.278 $\pm$ 0.024} \\
% \hline
\our{}  &   \bf{0.743 $\pm$  0.004} & \bf{ 0.270 $\pm$ 0.005} & \bf 0.299 $\pm$ 0.033 \\
\hline
\end{tabular}
% }
\end{center}
\label{tab:comp}
\end{table}

\paragraph{Results.} The results presented in Table \ref{tab:comp} show that for the first two datasets, \our{} outperforms all tested models by a significant margin, i.e. the difference between \our{} and other methods is statistically significant. In the case of FreeSolv dataset the mean results obtained by \our{} is slightly worse than MAT and pos-GCN, but this difference is not statistically significant due to the relatively high variance. We emphasize that FreeSolv is extremely tiny dataset with only 513 examples, which makes it difficult to reliably compare the methods. 

It is evident from the experiments that including positional features consistently improves the performance of the models across all tasks. For the smallest dataset, FreeSolv, pos-GCN even surpasses the score of \our{}. Nevertheless, learning from bigger datasets requires a better way of managing positional data, which can be noted for ESOL and BBBP datasets for which pos-GCN performs significantly worse than \our{} but still better than the vanilla~GC.

\begin{figure}
\normalsize
\begin{center}
\includegraphics[width=\textwidth]{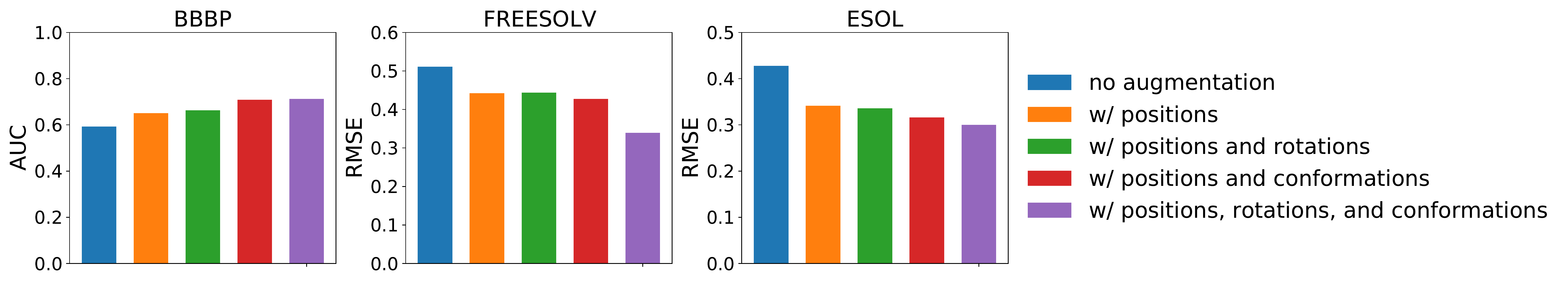} 
\end{center}
\caption{Comparison of different augmentation strategies on three chemical datasets. No augmentation is a pure GCN without positions. In the conformation variant multiple conformations were precalculated and then sampled during training. Rotation augmentation randomly rotates molecules in batches. For the first bar-plot higher is better, for the second and the third lower is better.}
\label{fig:conformations_rotations}
\end{figure}

\begin{figure}
% \normalsize
\centering
\includegraphics[width=0.8\textwidth]{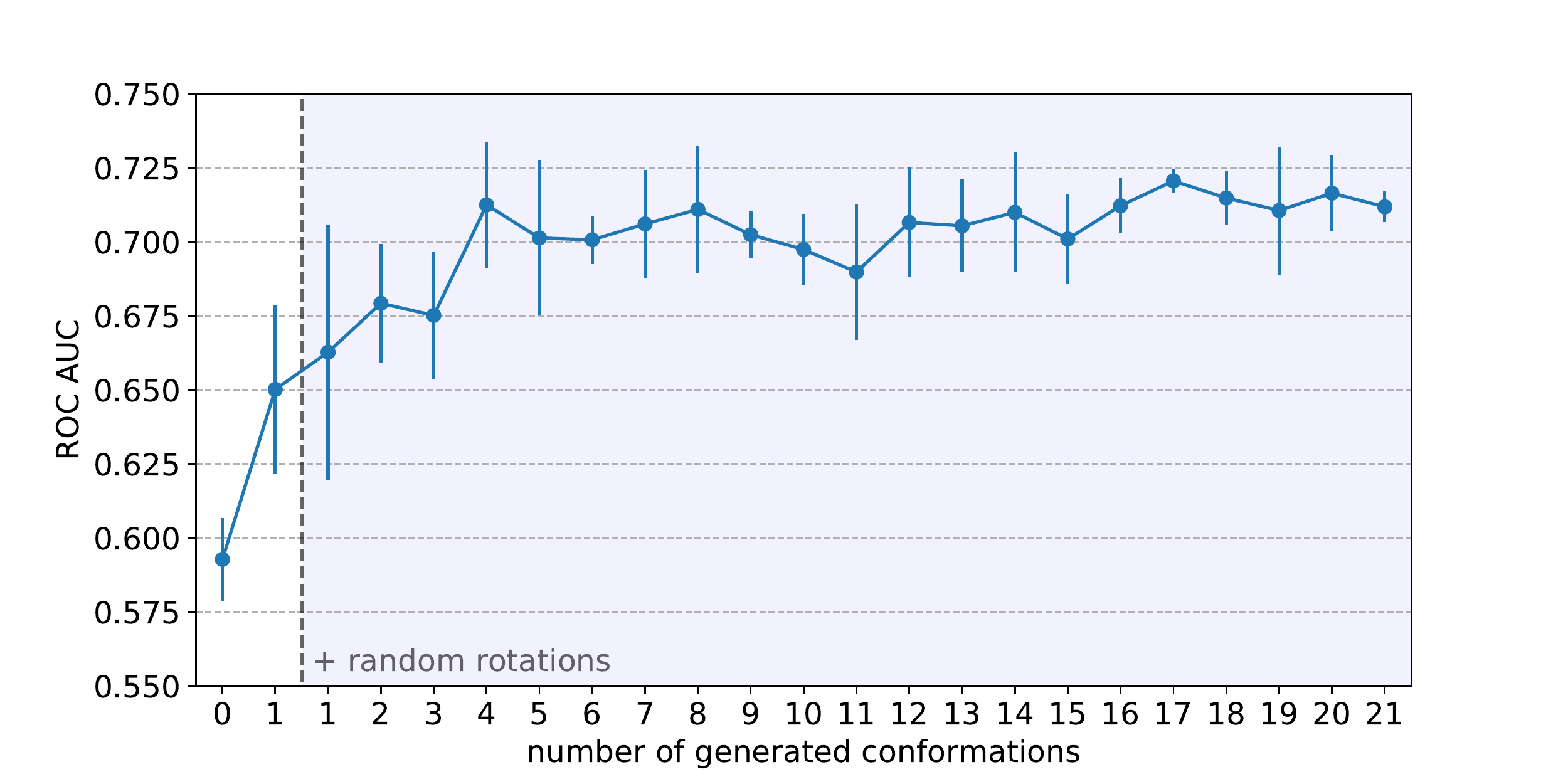} 
% \end{center}
\caption{ROC AUC scores achieved on the BBBP dataset by models using a different quantity of data augmentation. The first two models on the left (before the dashed line) are a standard GCN without node positions and \our{} with only one conformation (without any data augmentation) respectively. The models on the right of the dashed line are augmented with random rotations of a molecule. The amount of augmentation increases from left to right.}
\label{fig:conformations}
\end{figure}

\paragraph{Ablation study of the data augmentation.} In the case of CNNs, data augmentation is often used to extend the dataset through random image transformations, but for GCNs there are no extensive studies on data augmentation as they would involve graph transformations of some sort. In contrast to classical GCNs, the proposed \our{} allows to introduce data augmentation to enlarge the given dataset by transforming the geometry of a graph. In this experiment, we investigate the influence of data augmentation on the \our{} performance. 

First, we examined how removing predicted positions, and thus setting all positional vectors to zero in Equation~\ref{eq:neigh-our}, affects the scores achieved by our model on chemical tasks. The results are depicted in Figure~\ref{fig:conformations_rotations}. It clearly shows that even predicted node coordinates improve the performance of the method. On the same plot we also show the outcome of augmenting the data with random rotations and 30 predicted molecule conformations, which were calculated as described above. It occurs that the best performing model uses all types of position augmentation.

Eventually, we study the impact of various levels of augmentation. For this purpose, we precalculate 20 molecular conformations on the BBBP dataset using the UFF method and use them to augment the dataset. To test the importance of conformation variety, each run the number of available conformations to sample from is increased. The results are presented in Figure~\ref{fig:conformations}. One can see that including a bigger number of conformations helps the model to achieve better results. Also, the curve flattens out after a few conformations, which may be caused by the limited flexibility of small compounds and high similarity of the predicted shapes. It should be noted that data augmentation brings a huge improvement to the model, with more than 0.06 increase of the ROC AUC, which is enough to beat other models in the benchmark presented in the previous section. 

\section{Conclusion}

We proposed \our{} which is a general model for processing graph-structured data with spatial features. Node positions are integrated into our convolution operation to create a layer which generalizes both GCNs and CNNs. In contrast to the majority of other approaches, our method can effectively use added information about location to construct self-taught feature masking, which can be augmented to achieve invariance with respect to the desired properties. Furthermore, we provide a theoretical analysis of our spatial graph convolutions. The experiments confirm the strong performance of our method. 

It is also apparent in the benchmarks that both the local spatial features and graph-induced dependencies are needed to fully capture the nature of data in the presented setups. In the experiments, it was presented that not only does the information about conformations benefit the training, but also the means of processing the spatial information of local neighborhood are crucial to a strong performance.  

\bibliographystyle{splncs04}
\bibliography{ref}

\end{document}